\documentclass[twoside,11pt]{article}

\usepackage[linesnumbered,ruled,vlined,boxed]{algorithm2e}
\usepackage{enumerate}
\usepackage{algorithmic}
\usepackage{amsmath, amsthm, amssymb, fullpage}
\usepackage{graphicx}
\usepackage{caption}
\usepackage{subcaption}
\usepackage{verbatim} % This is for comment
\newtheorem{theorem}{Theorem}[section]
\newtheorem{definition}{Definition}

\newtheorem{corollary}{Corollary}

\newtheorem{lemma}{Lemma}

\title{\textbf{Weighted Clustering}}

\author{\textbf{Margareta Ackerman}\\
       Department of Computer Science\\
       Florida State University\footnote{E-mail: \texttt{mackerman@fsu.edu}}
       \and
       \textbf{Shai Ben-David} \\
       School of Computer Science\\
       University of Waterloo\footnote{E-mail: \texttt{shai@cs.uwaterloo.ca}}
       \and
       \textbf{Simina Br\^{a}nzei} \\
       Department of Computer Science and Engineering\\	
       Hebrew University of Jerusalem\footnote{E-mail: \texttt{simina.branzei@gmail.com}}
       \and
       \textbf{David Loker} \\
       School of Computer Science\\
       University of Waterloo\footnote{E-mail: \texttt{dloker@cs.uwaterloo.ca}}}

\begin{document}

\maketitle

%\vspace{3mm}

\begin{abstract}

One of the most prominent challenges in clustering is ``the user's dilemma," which is the problem of selecting an appropriate clustering algorithm for a specific task. A formal approach for addressing this problem relies on the identification of succinct, user-friendly properties that formally capture when certain clustering methods are preferred over others. 

Until now these properties focused on advantages of classical Linkage-Based algorithms, failing to identify when other clustering paradigms, such as popular center-based methods, are preferable. We present surprisingly simple new properties that delineate the differences between common clustering paradigms, which  clearly and formally demonstrates advantages of center-based approaches for some applications. These properties address how sensitive algorithms are to changes in element frequencies, which we capture in a generalized setting where every element is associated with a real-valued weight.

%To this end, we study a natural generalization of the classical clustering setting by associating every element with a real-valued weight. We conduct the first extensive theoretical analysis on the influence of weighted data on standard clustering algorithms in both the partitional and hierarchical settings, leading to three succinct, user-friendly, properties that can aid in the user's dilemma. 

% Our analysis raises several interesting questions and can be directly mapped to the classical unweighted setting.
\end{abstract}

\newpage

\section{Introduction}

Although clustering is  one of the most useful data mining tools, it suffers from a substantial disconnect between theory and practice. Clustering is applied in a wide range of disciplines, from astronomy to zoology, yet its theoretical underpinnings are still poorly understood. Even the fairy basic problem of which algorithm to select for a given application (known as ``the user's dilemma'')  is left to ad hoc solutions, as theory is only starting to address fundamental differences between clustering methods~(\cite{COLT2010,reza,NIPS2010,IJCAI2011}). Indeed, issues of running time complexity and space usage are still the primary considerations when choosing clustering techniques. Yet, for clustering, such considerations are inadequacy. Different clustering algorithms often produce radically different results on the same input, and as such, differences in their input-output behavior should take precedence over computational concerns.

``The user's dilemma," has been tackled since the 70s~(\cite{fisher,wright}), yet we still do not have an adequate solution.  A formal approach to this problem (see, for example, \cite{fisher, reza, NIPS2010}) proposes that we rely on succinct mathematical properties that reveal fundamental differences in the input-output behaviour of different clustering algorithms. However there is a serious shortcoming with the current state of this literature. Virtually all the properties proposed in this framework highlight the advantages of linkage-based methods, most of which are satisfied by single-linkage -- an algorithm that often performs poorly in practice. If one were to rely on existing properties to try to select a clustering algorithm, they would inevitably select a linkage-based technique. According to these properties, there is never a reason to choose, say, algorithms based on the  $k$-means objective (\cite{steinley2006k}), which often performs well in practice. 

Of course, practitioners of clustering have known for a long time that, for many applications, variations of the $k$-means method outperform classical linkage-based techniques. Yet a lack of clarity as to why this is the case leaves the ``the user's dilemma" largely unsolved. Despite continued efforts to find better clustering methods, the ambiguous nature of clustering precludes the existence of a single algorithm that will be suited for all applications. As such, generally successful methods, such as popular algorithms for the $k$-means objective, are ill-suited for some applications. To this end, it is necessary for users of clustering to understand how clustering paradigms differ in their input-output behavior. 

Unfortunately, informal recommendations are not sufficient. Many such recommendations advise to use $k$-means when the true clusters are spherical and to apply single-linkage when they may possess arbitrary shape. 
%%% simina::: Commented out yet unfortunately because it was repeated in the previous sentences
Such advice can be misguiding, as clustering users know that single-linkage can fail to detect arbitrary-shaped clusters, and $k$-means does not always succeed when clusters are spherical. Further insight comes from viewing data as a mixture model (when variations of $k$-means, particularly EM, are known to perform well), but unfortunately most clustering users simply don't know how their data is generated. Another common way to differentiate clustering methods is to partition them into  partitional and hierarchical and to imply that users should choose algorithms based on this consideration. Although the format of the output is important, is does not go to the heart of the matter, as most clustering approaches can be expressed in  both frameworks.\footnote{For example, $k$-means can be reconfigured to output a dendrogram using Ward's method and Bisecting $k$-mean, and classical hierarchical methods can be terminated using a variety of termination conditions (\cite{Kleinberg}) to obtain a single partition instead of a dendrogram.}

\emph{The lack of formal understanding of the key differences between clustering paradigms leaves users at a loss when selecting algorithms}. In practice, many users give up on clustering altogether when a single algorithm that had been successful on a different data set fails to attain a satisfactory clustering on the current data. Not realizing the degree to which algorithms differ, and the ways in which they differ, often prevents users from selecting appropriate algorithms, or even sampling a few diverse methods. 

This, of course, need not be the case. A set of simple, succinct properties can go a long way towards differentiating between clustering techniques and assisting users when choosing a method. As mentioned earlier, the set of previously proposed properties is inadequacy. This paper  identifies the first set of properties that differentiates between some of the most popular clustering methods, while highlighting potential advantages of $k$-means (and similar) methods.  The properties are very simple, and go to the heart of the difference between some clustering methods. However, the reader should keep in mind that they are not necessarily sufficient, and that in order to have a complete solution to ``the user's dilemma" we need additional properties that identify other ways in which clustering techniques differ. The ultimate goal is to have a small set of complementary properties that together aid in the selection of clustering techniques for a wide range of applications.

The properties proposed in this paper center around the rather basic concept of how different clustering methods react to element duplication. This leads to three surprisingly simple categories, each highlighting when some clustering paradigms should be used over others. To this end, we consider a generalization of the notion of element duplication by casting the clustering problem in the weighted setting, where each element is associated with a real valued weight. Instances in the classical model can be readily mapped to the weighted framework by replacing duplicates with integer weights representing the number of occurrences of each data point. 

This generalized setting enables more accurate representation of some clustering instances. Consider, for instance, vector quantification, which aims to find a compact encoding of signals that has low expected distortion. The accuracy of the encoding is most important for signals that occur frequently. With weighted data, such a consideration is easily captured by having the weights of the points represent signal frequencies. When applying clustering to facility allocation, such as the placement of police stations in a new district, the distribution of the stations should enable quick access to most areas in the district. However, the accessibility of different landmarks to a station may have varying importance. The weighted setting enables a convenient method for prioritizing certain landmarks over others.

%Traditional clustering algorithms can be readily translated into the weighted setting by considering their behaviour on data that includes element duplicates. 
%This leads to the following fundamental question: Given a specific weighted clustering task, how should a user select an algorithm for that task?

We formulate intuitive properties that may allow a user to select an algorithm based on how it treats weighted data (or, element duplicates).  These surprisingly simple properties are able to distinguish between classes of clustering techniques and clearly delineate instances in which some methods are preferred over others, without having to resort to assumptions about how the data may have been generated. As such, they may aid in the clustering selection process for clustering users at all levels of expertise.

Based on these properties we obtain a classification of clustering algorithms into three categories: those that are affected by weights on all data sets, those that ignore weights, and
those methods that respond to weights on some configurations of the data but not on others. Among the methods that always respond to weights
are several well-known algorithms, such as $k$-means and $k$-median. On the other hand, algorithms such as single-linkage, complete-linkage, and min-diameter ignore weights.

%In the unweighed setting, a user would utilize this classification by considering a %simple question: For the given application, should the algorithm be effected by %element duplication? 

From a theoretical perspective, perhaps the most notable is the last category. We find that methods belonging to that category are robust to weights when data is sufficiently clusterable, and respond to weights otherwise. Average-linkage as well as the well-known spectral objective function, ratio cut,
both fall into this category. We characterize the precise conditions under which these methods are influenced by weights.

%In the unweighed setting, a user would utilize this classification by %considering a simple question: For the given application, should the algorithm %be effected by element duplication? For instance, if the data has been %obtained through a biased sampling procedure, then it may be preferable to %avoid methods that are sensitive to weights (or, similarity, sensitive to the %duplication of elements). 

%Our analysis also reveals the following interesting phenomenon: algorithms that are known to perform well in practice %(in the classical, unweighted setting), tend to be more responsive to weights. For example, k-means is highly %responsive to weights while single linkage, which often performs poorly in practice \cite{Hartigan}, is weight robust.

\subsection{Related Work}

Clustering algorithms are usually analyzes in the context of unweighted data. The weighted clustering framework was briefly considered in the early 70s, but wasn't developed further until now. 
\cite{fisher} introduced several properties of clustering algorithms. Among these, they include ``point proportion admissibility'', which requires that the output of an algorithm should not change if any points are
duplicated. They then observe that a few algorithms are point proportion admissible.
However, clustering algorithms can display a much wider range of behaviours on weighted data than merely satisfying or failing to satisfy point
proportion admissibility. We carry out the first extensive analysis of clustering on weighted data, characterizing the precise conditions under which
algorithms respond to weight.

In addition, \cite{wright} proposed a formalization of cluster analysis consisting of eleven axioms. In two of these axioms, the
notion of mass is mentioned. Namely, that points with zero mass can be treated as non-existent,
and that multiple points with mass at the same location are equivalent to one point with weight the sum of the masses. The idea of mass has not been developed beyond stating these axioms in their work.

Like earlier work, recent work on simple properties capturing differences in the input-output behaviour of clustering methods also focuses on the unweighed partitional (\cite{COLT2010,reza,NIPS2010, Kleinberg}) and hierarchical settings (\cite{IJCAI2011}). This is the first application of this property-based framework to weighted clustering. 

Lastly,  previous work in this line of research centers on classical linkage-based methods and their advantages. Particularly well-studied is the single-linkage algorithm, for which there are multiple property-based characterizations, showing that single-linkage is the unique algorithm that satisfies several sets of properties  (\cite{jardine1968construction, reza, carlsson2010characterization}). More recently, the entire family of linkage-based algorithms was characterized (\cite{COLT2010, IJCAI2011}), differentiating those algorithms from other clustering paradigms by presenting some of the advantages of those methods.  In addition, previous property-based taxonomies in this line of work highlight the advantages of linkage-based methods (\cite{NIPS2010,fisher}), and some early work focuses on properties that distinguish among linkage-based algorithms (\cite{hubert1975hierarchical}). Despite the emphasis on linkage-based methods in the theory literature,  empirical studies and user experience have shown that, in many cases, other techniques produce more useful clusterings than those obtained by classical linkage-based methods. Here we propose categories that distinguish between clustering paradigms while also showing when other techniques, such as popular center-based methods, may be more appropriate.

\section{Preliminaries}
A \emph{weight function} $w$ over $X$ is a function $w: X \rightarrow R^+$, mapping elements of $X$ to positive real numbers. Given a domain set $X$,
denote the corresponding weighted domain by $w[X]$, thereby associating each element $x \in X$ with weight $w(x)$.
A \emph{dissimilarity function} is a symmetric function $d: X \times X \rightarrow R^+\cup\{0\}$, such that $d(x,y) = 0$ if and only if $x=y$. We
consider \emph{weighted data sets} of the form $(w[X],d)$, where $X$ is some finite domain set, $d$ is a dissimilarity function over $X$, and $w$ is a weight function over $X$.

A \emph{k-clustering} $C = \{C_1, C_2, \ldots, C_k\}$ of a domain set $X$ is a
partition of $X$ into $1 < k < |X|$ disjoint, non-empty subsets of $X$ where $\cup_{i} C_i =
X$.  A \emph{clustering} of $X$ is a $k$-clustering for some $1 < k < |X|$.
To avoid trivial partitions, clusterings that consist of a single cluster, or where every cluster has a unique element, are not permitted.

Denote the \emph{weight of a cluster} $C_i \in C$ by $w(C_i) = \sum_{x \in C_i}w(x)$.
For a clustering $C$, let $|C|$ denote the number of clusters in $C$.
For $x,y
\in X$ and clustering $C$ of $X$, write $x \sim_C y$ if $x$ and $y$ belong
to the same cluster in $C$ and $x \not\sim_C y$, otherwise.

A \emph{partitional weighted clustering algorithm} is a function that maps a data set $(w[X], d)$ and an integer $1 < k < |X|$ to a $k$-clustering of $X$.

A \emph{dendrogram} $\mathcal{D}$ of $X$ is a pair $(T, M)$ where $T$ is a strictly binary rooted tree and $M:
leaves(T) \rightarrow X$ is a bijection. A \emph{hierarchical weighted clustering algorithm} is a function
that maps a data set $(w[X], d)$ to a dendrogram of $X$. 
A set $C_0 \subseteq X$ is a cluster in  a dendrogram $\mathcal{D} = (T,M)$ of $X$ if there exists a node
$x$ in $T$ so that $C_0 = \{M(y) \mid y \textrm{ is a leaf and a descendent of }x\}$. Two dendrogram of $X$ are \emph{equivalent} if they contain the same clusters, and $[\mathcal{D}]$ denotes the equivalence class of dendrogram $\mathcal{D}$. 

For a hierarchical weighted clustering
algorithm $\mathcal{A}$, a clustering $C = \{C_1, \ldots,
C_k\}$ \emph{appears in} $\mathcal{A}(w[X],d)$ if $C_i$ is a cluster in $\mathcal{A}(w[X],d)$ for all $1 \leq i \leq k$. A partitional algorithm $\mathcal{A}$ outputs clustering $C$ on $(w[X],d)$ if $\mathcal{A}(w[X],d,|C|)= C$.

For the remainder of
this paper, unless otherwise stated, we will use the term ``clustering algorithm" for ``weighted clustering algorithm".

The range of a partitional algorithm on a data set is the number of clusterings it outputs on that data over all weight functions.

\begin{definition}[Range (Partitional)]
Finally, given a partitional clustering algorithm $\mathcal{A}$, a data set $(X,d)$, and $1 \leq k \leq |X|$, let $range(\mathcal{A}(X,d,k)) = \{ C \mid \exists w \textrm{ such that } C= \mathcal{A}(w[X],d)\},$ i.e. the set of $k$-clusterings that $\mathcal{A}$ outputs on $(X,d)$ over all possible weight
functions.
\end{definition}

The range of a hierarchical algorithm on a data set is the number of equivalence classes it outputs on that data over all weight functions.

\begin{definition}[Range (Hierarchical)]
Given a hierarchical clustering algorithm $\mathcal{A}$ and a data set $(X,d)$, let $range(\mathcal{A}(X,d)) = \{[\mathcal{D}] \mid \exists w \textrm{ such that } \mathcal{D}= \mathcal{A}(w[X],d)\},$
i.e. the set of dendrograms that $\mathcal{A}$ outputs on $(X,d)$ over all possible weight
functions.
\end{definition}

%%%%%%%%%%%%%%%%%%%%%%%%%%%%%%%%%%%%%%%%%%%%%%%%%%%%%%%%%%%%%%%%%%%%%%%%%%%%%%%%%%%%%%%%%%%%%%%%%%%%%%%%%%%%%%%%%%%%%%%%%%%%%%%%%%%%%%
%%%%%%%%%%%%%%%%%%%%%%%%%%%%%%%%%%%%%%%%%%%%%%%%%%%%%%%%%%%%%%%%%%%%%%%%%%%%%%%%%%%%%%%%%%%%%%%%%%%%%%%%%%%%%%%%%%%%%%%%%%%%%%%%%%%%%%
%%%%%%%%%%%%%%%%%%%%%%%%%%%%%%%%%%%%%%%%%%%%%%%%%%%%%%%%%%%%%%%%%%%%%%%%%%%%%%%%%%%%%%%%%%%%%%%%%%%%%%%%%%%%%%%%%%%%%%%%%%%%%%%%%%%%%%

\section{Basic Categories}\label{basic}

Different clustering algorithms exhibit radically different response to weighted data. In this section we introduce a
formal categorization of clustering algorithms based on their response to weights. This categorization identifies fundamental differences between clustering paradigms, while highlighting when some of the more empirically successful methods should be used. These simple properties can assist clustering users in selecting suitable method by simply considering how an appropriated algorithm should react to element duplication. After we introduce the three categories, we show a classification of some of well-known clustering methods according to their response to weight, summarized in Table 1.

%First, we define what it means for a partitional algorithm to be weight responsive on a %clustering.

%We present an analogous definition for hierarchical algorithms when we study hierarchical %algorithms below.

%\begin{definition}[Weight responsive]\label{weight responsive}
%A partitional clustering algorithm $\mathcal{A}$ is \emph{weight-responsive} on a clustering $C$ %of $(X,d)$ if
%\begin{enumerate}
%\item there exists a weight function $w$ so that $\mathcal{A}(w[X], d) = C$, and
%\item there exists a weight function $w'$ so that $\mathcal{A}(w'[X], d) \neq C$.
%\end{enumerate}
%\end{definition}

\subsection{Weight Robust Algorithms}

We first introduce the notion of ``weight robust" algorithms. Weight robustness requires that the output of the algorithm be unaffected by changes of element weights (or, the number of occurrences of each point in the unweighted setting). This category is closely related to ``point proportion admissibility'' by \cite{fisher}. 

\begin{definition}[Weight Robust (Partitional)]
A partitional algorithm $\mathcal{A}$ is \emph{weight-robust} if for all $(X,d)$ and $1 < k < |X|$, $|range(\mathcal{A}(X,d,k))| =1$.
\end{definition}

The definition in the hierarchical setting is analogous.

\begin{definition}[Weight Robust (Hierarchical)]
A hierarchical algorithm $\mathcal{A}$ is \emph{weight-robust} if for all $(X,d)$, $|range(\mathcal{A}(X,d))| =1$.
\end{definition}

At first glance, this appears to be a desirable property. A weight robust algorithm is able to keep sight on the geometry of the data without being ``distracted" by  weights, or element duplicates. Indeed, when a similar property was proposed by \cite{fisher}, it was presented as a desirable characteristic. 

Yet, notably, few algorithms possess it (particularly single-linkage, complete-linkage, and min-diamater), while most techniques, including those with a long history of empirical success, fail this property. This brings into question how often is weight-robustness a desirable characteristic, and suggests that at least for some application  sensitivity to weights may be an advantage. 
Significantly, the popular $k$-means and similar methods fail weight robustness in a strong sense, being ``weight sensitive." 

\subsection{Weight Sensitive Algorithms}

We now introduce the definition of ``weight sensitive'' algorithms. 

\begin{definition}[Weight Sensitive (Partitional)]
A partitional algorithm $\mathcal{A}$ is \emph{weight-sensitive} if for all $(X,d)$ and $1 < k < |X|$, $|range(\mathcal{A}(X,d,k))| > 1$.
\end{definition}

The definition is analogous for hierarchical algorithms.

\begin{definition}[Weight Sensitive (Hierarchical)]
A hierarchical algorithm $\mathcal{A}$ is \emph{weight-sensitive} if for all $(X,d)$ where $|X|>2$, $|range(\mathcal{A}(X,d))| > 1$.
\end{definition}

Note that this definition is quite extreme. It means that \emph{no matter how well-separated} the clusters are, the output of a weight-sensitive algorithm can be altered by modifying some of the weights. That is, a weight-sensitive algorithm will miss arbitrarily well-separated clusters, for some weighting of its elements. In practice,  weight sensitive algorithm tend to aim for balanced cluster sizes, and so prioritize a balance in cluster sizes (or, sum of cluster weights) over separation between clusters. 

While weight-robust algorithms are interested exclusively in the geometry of the data, weight-sensitive techniques have two potentially conflicting considerations: The weight of the points and the geometry of the data. For instance, consider the data in Figure~\ref{fig:different}, which has two distinct 3-clusterings, one which provides superior separation between clusters, and another in which clusters sizes are balanced. Note how different are the two clusterings from one each other. All the weight-sensitive methods we consider select the clustering on the right, as it offers more balanced clusters. On the other hand, the weight-robust methods we studied picked the clustering on the left hand side, as it offers better cluster separation. 

Another way we could think of weight sensitive algorithm is that, unlike weight-robust methods, weight sensitive algorithms allow the weights to alter the geometry of the data. In contrast, weight robust techniques do not allow the weights of the points to ``interfere'' with the underlying geometry. 

It is important to note that there appear to be implications of these categories that apply to data that is neither weighted nor contains element duplicates. Considering the algorithms we analyzed (summarized in Table 1), the behaviour we observe on element duplicates extend to ``near-duplicates,'' which are closely positioned elements. Furthermore, the weight response of an algorithm sheds light on how it treats dense regions. In particular, weight sensitive algorithms have a tendency to "zoom in" on areas of high density, effectively ignoring sparse regions, as shown on the right-hand side of Figure~\ref{fig:different}. 

Finally, the last category considered here offers a compromise between weight-robustness and weight-sensitivity, we refer to this category as \emph{weight considering}.

\subsection{Weight Considering Algorithms}

\begin{definition}[Weight Considering (Partitional)]
A partitional algorithm $\mathcal{A}$ is \emph{weight-considering} if
\begin{itemize}
\item There exist $(X,d)$ and $1< k <|X|$ so that $|range(\mathcal{A}(X,d,k))| = 1$, and
\item There exist $(X,d)$ and $1< k <|X|$ so that $|range(\mathcal{A}(X,d,k))| > 1$.
\end{itemize}
\end{definition}

The definition carries over to the hierarchical setting as follows.

\begin{definition}[Weight Considering (Hierarchical)]
A hierarchical algorithm $\mathcal{A}$ is \emph{weight-considering} if
\begin{itemize}
\item There exist $(X,d)$ with $|X|>2$ so that $|range(\mathcal{A}(X,d))| = 1$, and
\item There exist $(X,d)$ with $|X|>2$ so that $|range(\mathcal{A}(X,d))| > 1$.
\end{itemize}
\end{definition}

Weight considering methods appear to have the best of both worlds. The weight-considering algorithms that we analyzed (average-linkage and ratio-cut), ignore weights when clusters are sufficiently well-separated and otherwise takes them into consideration. Yet, it is important to note that this is only desirable in some instances. For example, when cluster balance is critical, as may be the case for market segmentation, weight-sensitive methods may be preferable over weight considering ones. On the other hand, when the distribution may be highly bias, as is the often the case for phylogenetic analysis, weight-considering methods may offer a satisfactory compromise between weight-sensitivity and weight-robustness, allowing the algorithm to detect well-separated, possibly of radically varying sized, without entirely disregarding weights. Notably, of all the classical clustering algorithms studied here, average-linkage, a weight-considering technique, is the only one that is commonly applied to phylogenetic analysis. 

\begin{figure}
\begin{center}
\includegraphics[scale=0.9]{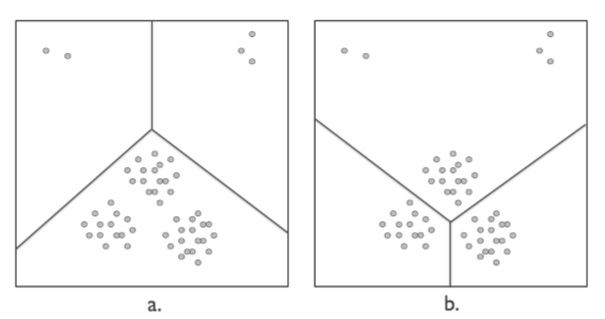} 
\end{center}
\caption{An example of different cluster structures in the same data, illustrating inherent tradeoffs between separation and cluster balance. The clustering on the left finds inherent structure in the data by identifying well-separated partitions, while the clustering on the right discovers structure in the data by focusing on the dense region, achieving more balanced cluster sizes. The correct partitioning depends on the application at hand. }
\label{fig:different}
\end{figure}

The following  table presents a classification of classical clustering methods based on these three categories. Sections~\ref{functions} and \ref{sec:Hierachical} provide the proof for the results summarized below. In addition, these sections also characterizes precisely when the weight-considering techniques studied here respond to weights. An expanded table that includes heuristics, with a corresponding analysis, is included in Section~\ref{heuristics}.

\begin{table*}[ht]
\begin{center}
\begin{tabular}[ht!]{||l||l|l|l||}
\hline
\hline
 & {\textbf{Partitional}}  & {\textbf{Hierarchical}} \\
\hline
\hline
\textbf{Weight} &  $k$-means, $k$-medoids& Ward's method\\
\textbf{Sensitive}& $k$-median, Min-sum& Bisecting $k$-means \\
%&\emph{Exemplar-based}  & \emph{$\mathcal{F}$-Divisive}\footnote{$\mathcal{F}$-Divisive is weight-sensitive whenever $\mathcal{F}$ is weight-separable.} & \\
 \hline
\textbf{Weight} &  &   \\
\textbf{Considering} & Ratio-cut & Average-linkage \\
%  &  &  & \emph{Deterministic Lloyd}\footnote{Deterministic Lloyd is weight considering whenever the initialization is weight-independent.}\\
 \hline
\textbf{Weight} & Min-diameter & Single-linkage   \\
\textbf{Robust}   & $k$-center & Complete-linkage \\
\hline
\hline
\end{tabular}
\end{center}
\caption{A classification of clustering algorithms based on their response to weighted data.}
\label{summary}
\end{table*}

To formulate clustering algorithms in the weighted setting, we consider their behaviour on
data that allows duplicates. Given a data set $(X,d)$, elements $x, y \in X$ are
duplicates if $d(x,y) = 0$ and $d(x,z) = d(y, z)$ for all $z \in X$. In a Euclidean space, duplicates
correspond to elements that occur at the same location. We obtain the weighted version of a data set
by de-duplicating the data, and associating every element with a weight equaling the number of
duplicates of that element in the original data. The weighted version of an algorithm partitions the
resulting weighted data in the same manner that the unweighted version partitions the original data.
As shown throughout the paper, this translation leads to natural formulations of weighted
algorithms.

%%%%%%%%%%%%%%%%%%%%%%%%%%%%%%%%%%%%%%%%%%%%%%%%%%%
%%%%%%%%%%%%%%%%%%%%%%%%%%%%%%%%%%%%%%%%%%%%%%%%%%%

\section{Partitional Methods}\label{functions}

In this section, we show that partitional clustering algorithms respond to weights in a variety of ways. Many popular partitional clustering paradigms, including $k$-means, $k$-median, and
min-sum, are weight sensitive. It is easy to see that methods such as min-diameter and $k$-center are weight-robust. We begin by analysing the behaviour of a spectral objective function ratio cut, which exhibits interesting behaviour on weighted data by responding to weight unless data is highly structured.

%%%%%%%%%%%%%%%%%%%%%%%%%%%%%%%%%%%%%%%%%%%%%%%%%%%%%%%%%%%%%%%%%%%%%%%%%%%%%%%%%%%%%%%%%%%%%%%%%%%%%%%%%%%%%%%%%%%%%%%%%%%%%%%%%%%%%%
%%%%%%%%%%%%%%%%%%%%%%%%%%%%%%%%%%%%%%%%%%%%%%%%%%%%%%%%%%%%%%%%%%%%%%%%%%%%%%%%%%%%%%%%%%%%%%%%%%%%%%%%%%%%%%%%%%%%%%%%%%%%%%%%%%%%%%
\subsection{Ratio-Cut Clustering}\label{sec:ratio}

We investigate the behaviour of a spectral objective function, ratio-cut (\cite{von2007tutorial}), on weighted data. Instead of a dissimilarity function, spectral clustering relies on a \emph{similarity function}, which maps pairs of domain elements to non-negative real numbers that represent how alike the elements are.
The ratio-cut of a clustering $C$ is:
\[
\textrm{cost}_{rcut}(C, w[X],s) = \frac{1}{2} \sum_{C_i \in C} \frac{\sum_{x \in C_i, y \in X \backslash C_i} s(x, y) \cdot w(x) \cdot w(y)}{\sum_{x \in
C_i}w(x)}.
\]

The ratio-cut clustering function is:
\[
\textrm{rcut}(w[X],s,k) = \arg \min_{C;|C|=k}\textrm{cost}_{rcut}(C,w[X],s).
\]
We prove that this function ignores data weights only when the data satisfies a very strict notion of clusterability. To characterize precisely when ratio-cut responds to weights, we first present a few definitions.

A clustering $C$ of $(w[X],s)$ is \emph{perfect} if for all $x_1,x_2,x_3,x_4 \in X$ where $x_1 \sim_C x_2$ and $x_3 \not\sim_C x_4$, $s(x_1,x_2) > s(x_3,x_4)$. $C$ is \emph{separation-uniform} if there exists $\lambda$ so that for all $x, y \in X$ where $x \not\sim_C y$, $s(x,y) = \lambda$. Note that neither condition depends on the weight function.

We show that whenever a data set has a clustering that is both  perfect and separation-uniform, then
ratio-cut uncovers that clustering, which implies that ratio-cut is not weight-sensitive. Note that, in particular,
these conditions are satisfied when all between-cluster similarities are set to zero. On the other
hand, we show that ratio-cut does respond to weights when either condition fails.

\begin{lemma}\label{uniform}
If, given a data set $(X,d)$, $1<k<|X|$ and some weight function $w$, ratio-cut outputs a $k$-clustering $C$ that is not separation-uniform and where every cluster has more than a single point, then $|range(\textrm{ratio-cut}(X,d))| > 1$.
\end{lemma}
\begin{proof}
We consider two cases. \\

\textbf{Case 1}: There is a pair of clusters with different similarities between them. Then there exist $C_1, C_2 \in C$, $x \in C_1$, and $y \in C_2$ so that $s(x,y)
\geq s(x, z)$ for all $z \in C_2$, and there exists $a \in C_2$ so that $s(x,y)
> s(x, a)$.

Let $w$ be a weight function such that $w(x) = W$ for some sufficiently large
$W$ and weight $1$ is assigned to all other points in $X$. Since we can set $W$ to be arbitrarily large, when looking at the cost of a cluster,
it suffices to consider the dominant term in terms of $W$.
We will show that we can improve the cost of $C$ by moving a point from $C_2$ to $C_1$. Note
that moving a point from $C_2$ to $C_1$ does not affect the dominant term of clusters other than
$C_1$ and $C_2$. Therefore, we consider the cost of these two clusters before and after
rearranging points between these clusters.

Let $A  = \sum_{a \in
C_2}s(x,a)$ and let $m = |C_2|$. Then the dominant term, in terms of $W$, of the cost of $C_2$ is $W\left(\frac{A}{m}\right)$. The cost of $C_1$ approaches a constant as $W \to \infty$.

%All other terms in $rcut(C, w[X],s)$
%approach a constant as $W$ goes to infinity.

Now consider clustering $C'$ obtained from $C$ by moving $y$ from
cluster $C_2$ to cluster $C_1$. The dominant term in the cost of $C_2$ becomes $W\left(\frac{A - s(x,y)}{m-1}\right)$, and the cost of $C_1$
approaches a constant as $W \to \infty$. By choice of $x$ and $y$,
if $\frac{A - s(x,y)}{m-1} < \frac{A}{m}$ then $C'$ has lower loss than $C$ when
$W$ is large enough. The inequality $\frac{A - s(x,y)}{m-1} < \frac{A}{m}$ holds when $\frac{A}{m} < s(x,y)$, and the
latter holds by choice of $x$ and $y$. \\

\textbf{Case 2}: The similarities between every pair of clusters are the same.
However, there are clusters $C_1, C_2, C_3 \in C$, so that the
similarities between $C_1$ and $C_2$ are greater than the ones between
$C_1$ and $C_3$. Let $a$ and $b$ denote the similarities between $C_1, C_2$
and $C_1, C_3$, respectively.

%Let $a$ denote the similarities between $C_1$ and $C_2$, and
%$b$ those between $C_1$ and $C_3$.

Let $x \in C_1$ and $w$ a weight function, such that $w(x) = W$ for
large $W$, and weight $1$ is assigned to all other points in
$X$. The dominant term comes from clusters going into $C_1$, specifically edges that
include point $x$.  The dominant term of the contribution of cluster $C_3$ is
$Wb$ and the dominant term of the contribution of $C_2$ is $Wa$, totaling $Wa +
Wb$.

Now consider clustering $C'$ obtained from clustering $C$ by merging $C_1$ with
$C_2$, and splitting $C_3$ into two clusters (arbitrarily). The dominant term of
the clustering comes from clusters other than $C_1 \cup C_2$, and the cost of
clusters outside $C_1 \cup C_2 \cup C_3$ is unaffected. The dominant term of
the cost of the two clusters obtained by splitting $C_3$ is $Wb$ for each, for a
total of $2Wb$. However, the factor of $Wa$ that $C_2$ previously contributed is no
longer present. This replaces the coefficient of the dominant term from
$a + b$ to $2b$, which improved the cost of the clustering because $b<a$.
\end{proof}

\begin{lemma}\label{notPerfect}
If, given a data set $(X,d)$, $1 < k < |X|$, and some weight function $w$, ratio-cut outputs a clustering $C$ that is not perfect and where every cluster has more than a single point, then $|range(\textrm{ratio-cut}(X,d,k))| > 1$.
\end{lemma}
\begin{proof}
If $C$ is also not separation-uniform, then Lemma~\ref{uniform} can be applied, and so we can assume that $C$ is separation-uniform.
Then there exists a within-cluster similarity in $C$ that is smaller than
some between-cluster similarity, and all between cluster similarities are the
same.
Specifically, there exist clusters $C_1$ and $C_2$, such that all the similarities
between $C_1$ and $C_2$ are $a$, and there exist $x,y \in C_1$ such that $s(x,y) <
a$. Let $b = s(x,y)$.

Let $w$ be a weight function such that $w(x) = W$ for large
$W$, and weight $1$ is assigned to all other points in $X$. Then the dominant
term in the cost of $C_2$ is $Wa$, which comes from the cost of points in
cluster $C_2$ going to point $x$. The cost of $C_1$ approaches a constant as $W \rightarrow \infty$.

Consider the clustering $C'$ obtained from $C$ by completely re-arranging all
points in $C_1,C_2 \in C$ as follows:
\begin{enumerate}
\item Let $\{y,z\} = C_1'$ be one cluster for any $z \in C_2$.
\item Let $C_2' = (C_1 \cup C_2) \setminus C_1'$ be the new second cluster.
\end{enumerate}

The dominant term in of $C_1'$, which comes from
the cost of points in cluster $C_1'$ going to point $x$, is $\left(\frac{a+b}{2}\right)W$, which is
smaller than $Wa$ since $a> b$. Note that the cost of each cluster outside of
$C_1' \cup C_2'$ remains unchanged. Since $x \in C_1'$, the cost of $C_1'$ approaches a constant as $W \rightarrow \infty$.  Therefore, $\textrm{cost}_{rcut}(C',w[X],s)$ is smaller
than $\textrm{cost}_{rcut}(C,w[X],s)$ when $W$ is sufficiently large.
\end{proof}

\begin{lemma}\label{rcut clusterable}
Given any data set $(w[X],s)$ and $1 < k <|X|$ that has a perfect, separation-uniform $k$-clustering $C$, ratio-cut$(w[X],s,k) = C.$
\end{lemma}
\begin{proof}
Let $(w[X],s)$ be a weighted data set, with a perfect, separation-uniform clustering $C = \{C_1, \ldots, C_k\}$.
Recall that for any $Y \subseteq X$, $w(Y) = \sum_{y \in Y} w(y)$. Then:

\begin{align*}
&\textrm{cost}_{rcut}(C, w[X],s) = \frac{1}{2} \sum_{i=1}^{k} \frac{\sum_{x \in C_i} \sum_{y \in \overline{C_i}} s(x,y) w(x) w(y)}{\sum_{x \in C_i} w(x)}\\
& = \frac{1}{2} \sum_{i=1}^{k} \frac{\sum_{x \in C_i} \sum_{y \in \overline{C_i}} \lambda w(x) w(y)}{\sum_{x \in C_i}
w(x)} \\
& = \frac{\lambda}{2} \sum_{i=1}^{k} \frac{\sum_{y \in \overline{C_i}} w(y) \sum_{x \in C_i} w(x)}{\sum_{x \in C_i} w(x)}
= \frac{\lambda}{2} \sum_{i=1}^{k} \sum_{y \in \overline{C_i}} w(y)\\
& = \frac{\lambda}{2} \sum_{i=1}^{k} w(\overline{C_i}) = \frac{\delta}{2} \sum_{i=1}^{k} [w(X) - w(C_i)] \\
& = \frac{\lambda}{2} \left(k w(X) - \sum_{i=1}^{k}w(C_i)\right) = \frac{\lambda}{2} (k - 1) w(X).
\end{align*}

Consider any other clustering, $C^{'} = \{C_1^{'}, \ldots, C_{k}^{'}\} \neq C$. Since $C$ is both perfect and separation-uniform, all between-cluster similarities in $C$ are $\lambda$, and all within-cluster similarities are greater than $\lambda$. From here it follows that all pair-wise similarities in the data are at least $\lambda$.  Since $C'$ is a $k$-clustering different from $C$, it must differ from $C$ on at least one between-cluster edge, so that edge must be greater than $\lambda$.
Thus the cost of $C^{'}$ is:
\begin{align*}
& \textrm{cost}_{rcut}(C^{'},w[X],s) =
\frac{1}{2} \sum_{i=1}^{k} \frac{\sum_{x \in C_i^{'}} \sum_{y \in \overline{C_i^{'}}} s(x,y) w(x) w(y)}{\sum_{x \in C_i^{'}} w(x)}
\\
& > \frac{1}{2} \sum_{i=1}^{k} \frac{\sum_{x \in C_i^{'}} \sum_{y \in \overline{C_i^{'}}} \lambda w(x) w(y)}{\sum_{x \in C_i^{'}} w(x)}\\
&= \frac{\lambda}{2} (k - 1) w(X)
= \textrm{cost}_{rcut}(C).
\end{align*}
Thus clustering $C^{'}$ has a higher cost than $C$.
\end{proof}

%\begin{proof}
%Let $(w[X],s)$ be a weighted data set, with a perfect, separation-uniform clustering $C$.
%It can be shown that $\mbox{rcut}(C, w[X],s,k) = \frac{\lambda}{2} (k - 1) w(X)$ for some $\lambda$. Given any %other clustering $C^{'}$, it can be shown that
%$\mbox{rcut}(C^{'},  w[X],s,k) > \frac{\lambda}{2} (k - 1) w(X)$. Please see appendix for details.
%\end{proof}

It follows that ratio-cut responds to weights on all data sets except those where it is possible to obtain cluster separation that is both very large and highly uniform. This implies that ratio cut is highly unlikely to be unresponsive to weights in practice. 

Formally, we have the
following theorem, which gives sufficient conditions for when ratio-cut ignores weights, as well conditions that make this function respond to weights. 

\begin{theorem}
Given any $(X,d)$ and $1 < k < |X|$, 
\begin{enumerate}
\item if $(X,d)$ has a clustering that is both perfect and separation uniform, then $$|range(\textrm{Ratio-cut}(X,s,k))|=1,$$ and
\item  if $range(\textrm{Ratio-cut}(X,s,k))$ includes a clustering $C$ that is not perfect, not separation uniform, and has no singleton clusters, then $|range(\textrm{Ratio-cut}(X,s,k))|>1.$
\end{enumerate}
\end{theorem}
\begin{proof}
The result follows by Lemma~\ref{uniform}, Lemma~\ref{notPerfect}, and Lemma~\ref{rcut clusterable}.
\end{proof}

%Formally, the $k$-center objective function is
%$\arg \min_{T; |T|=k} \max_{C_i \in C; x \in C_i} d(x, t_i)$,
%where $T$ is a set of centers, and $t_i \in T$ is the center of cluster $C_i$.

\subsection{$K$-Means}
Many popular partitional clustering paradigms, including $k$-means (see~\cite{steinley2006k} for a detailed exposition of this popular objective function and related algorithms), $k$-median, and
the min-sum objective (\cite{sahni1976p}), are weight sensitive. Moreover, these algorithms satisfy a stronger condition. By modifying weights, we can make these algorithms separate any set of points. We call such algorithms
\emph{weight-separable}.

\begin{definition}[Weight Separable]
A partitional clustering algorithm $\mathcal{A}$ is \emph{weight-separable} if for any data set
$(X,d)$ and any $S \subset X$, where $2 \leq |S| \leq k$, there exists a
weight function $w$ so that $x \not\sim_{\mathcal{A}(w[X],d, k)} y$ for all distinct $x,y \in S$.
\end{definition}

Note that every weight-separable algorithm is also weight-sensitive.
\begin{lemma}\label{separable}
If a clustering algorithm $\mathcal{A}$ is weight-separable, then $\mathcal{A}$ is weight-sensitive.
\end{lemma}
\begin{proof}
Given any $(X,d)$ and weight function $w$ over $X$, let $C = \mathcal{A}(w[X],d,k)$. Select points $x$ and $y$ where $x \sim_C y$.
Since  $\mathcal{A}$ is weight-separable, there exists $w'$ so that $x
\not\sim_{\mathcal{A}(w'[X],d, k)} y$, and so $\mathcal{A}(w'[X],d,k) \neq C.$ It follows that for any $(X,d)$, $|range(\mathcal{A}(X,d))|>1.$
\end{proof}

$K$-means is perhaps the most popular clustering objective function, with cost: $$k\textrm{-means}(C,w[X],d) = \sum_{C_i \in C} \sum_{x \in C_i}{d(x, cnt(C_i)})^2,$$ where $cnt(C_i)$ denotes the center of mass of cluster $C_i$.
The $k$-means objective function finds a clustering with minimal $k$-means cost. We show that $k$-means is weight-separable, and thus also weight-sensitive.

\begin{theorem}\label{thm:kmeans_sep}
The $k$-means objective function is weight-separable.
\end{theorem}
\begin{proof}
Consider any $S \subseteq X$. Let $w$ be a weight function over $X$ where $w(x) = W$ if $x \in S$, for large $W$,
and $w(x) = 1$ otherwise. As shown by \cite{ostrovsky}, the $k$-means objective function is equivalent to
$$\frac{ \sum_{x,y\in C_i}d(x,y)^2\cdot w(x)\cdot w(y)}{w(C_i)}.$$

Let $m_1 = \min_{x,y \in X}d(x,y)^2 > 0$, $m_2 = \max_{x,y \in X}d(x,y)^2$, and $n = |X|$.
Consider any $k$-clustering $C$ where all the elements in $S$ belong to distinct clusters.
Then we have:
\[
k\textrm{-means}(C, w[X],d) < km_2\left(n+\frac{n^2}{W}\right).
\]
On the other hand, given any $k$-clustering $C'$ where at least two elements of $S$ appear in the same cluster, $k\textrm{-}means(C',w[X],d) \geq \frac{W^2m_1}{W+n}$.
Since
\[
\lim_{W \to \infty} \frac{k\textrm{-means}(C', w[X], d) }{k\textrm{-means}(C, w[X], d) } = \infty,
\]
$k$-means separates all the elements in $S$ for large enough $W$.
\end{proof}

Min-sum is another well known objective function and it minimizes the expression: $$\sum_{C_i \in C} \sum_{x,y \in C_i}d(x,y)\cdot w(x)\cdot
w(y).$$

\begin{theorem}\label{thm:min_sum_sep}
Min-sum is weight-separable.
\end{theorem}
\begin{proof}
Let $(X,d)$ be any data set and $1 <k<|X|$.
Consider any $S \subseteq X$ where $1 < |S| \leq k$. Let $w$ be a weight function over $X$ where $w(x) = W$ if $x \in S$, for large $W$,
and $w(x) = 1$ otherwise. Let $m_1 = \min_{x,y \in X}d(x,y)$ be the minimum dissimilarity in $(X,d)$, and let $m_2 = \max_{x,y \in X}d(x,y)$ be the maximum dissimilarity in $(X,d)$. 

Then the cost of any cluster that includes two elements of $S$ is a least $m_1W^2$, while the cost of a cluster that includes at most one element of $S$ is less than $m_2|X|(|X|+W)$. So when $W$ is large enough, selects a partition where no two elements of $S$ appear in the same cluster.
\end{proof}

Several other objective functions similar to $k$-means are also weight-separable. We show that $k$-median and $k$-medoids are weight sensitive by analysing center-based approaches that use exemplars from the data  as cluster centers (as opposed to any elements in the underlying space). Given a set $T \subseteq X$, define $C(T)$ to be the clustering obtained by assigning every element in $X$ to the closest element (``center'') in $T$.

Exemplar-based clustering is defined as follows.

\begin{definition}[Exemplar-based]
An algorithm $\mathcal{A}$ is \emph{exemplar-based} if there exists a function
$f: R^+ \rightarrow R^+$ such that for all $(w[X],d)$, $\mathcal{A}(w[X], d,k) = C(T)$ where
\[
T =\arg \min_{T \subset X;|T|=k} \sum_{x \in X, x \not\in T} w(x) f(\min_{y \in T}d(x,y)).
\]
\end{definition}
Note that when $f$ is the identity function, we obtain $k$-median, and when $f(x) = x^2$ we obtain $k$-medoids.

\begin{theorem}
Every exemplar-based clustering function is weight-separable.
\end{theorem}
\begin{proof}
Consider any $S \subseteq X$ with $|S|\leq k$. For all $x \in S$, set $w(x) = W$ for some large
value $W$ and set all other weights to $1$. Recall that a clustering has between $2$ and $|X|-1$
clusters.  Consider any clustering $C$ where some distinct elements $x,y \in S$ belong to the same
cluster $C_i \in C$. Since at most one of $x$ or $y$ can be the center of $C_i$, the cost of $C$ is
at least $W\cdot\min_{x_1,x_2 \in C_i}f(d(x_1,x_2))$. Observe that $f(d(x_1,x_2))>0$.

Consider any clustering $C'$ where all the elements in $S$ belong to distinct clusters. If a cluster contains a unique element $x$ of $S$, then its cost is constant in $W$ if $x$ is the cluster center, and at least $W\cdot \min_{x_1,x_2 \in X}f(d(x_1,x_2))$ if $x$ is not the center of the cluster. This shows that if $W$ is large enough, then every element of $S$ will be a cluster center, and so the cost of $C'$ would be independent of $W$. So when $W$ is large, clusterings that separate elements of $S$ have lower cost than those that merge any points in $S$.
\end{proof}

We have the following corollary.

%By Lemma~\ref{separable}, we have the following:
\begin{corollary}
The $k$-median and $k$-medoids objective functions are weight-separable and weight-sensitive.
\end{corollary}

%Observe that all of these popular objective functions are highly responsive to weight.

%This suggests that responsiveness to weight may be a desirable feature.

%We show that popular center-based algorithms are weight-separating.

%The proof is similar to that of Theorem~\ref{thm:kmeans_sep}.

%%%%%%%%%%%%%%%%%%%%%%%%%%%%%%%%%%%%%%%%%%%%%%%%%%%%%%%%%%%%%%%%%%%%%%%%%%%%%%%%%%%%%%%%%%%%%%%%%%%%%%%%%%%%%%%%%%%%%%%%%%%%%%%%%%%%%%
%%%%%%%%%%%%%%%%%%%%%%%%%%%%%%%%%%%%%%%%%%%%%%%%%%%%%%%%%%%%%%%%%%%%%%%%%%%%%%%%%%%%%%%%%%%%%%%%%%%%%%%%%%%%%%%%%%%%%%%%%%%%%%%%%%%%%%

\section{Hierarchical Algorithms}~\label{sec:Hierachical}

Similarly to partitional methods, hierarchical algorithms also exhibit
a wide range of responses to weights. We show that Ward's method (\cite{ward1963hierarchical}), a successful linkage-based algorithm, as well as popular divisive hierarchical methods, are weight sensitive. On the other hand, it is easy to see that the linkage-based algorithms single-linkage and complete-linkage are both weight robust, as was observed in \cite{fisher}.

Average-linkage, another popular linkage-based method, exhibits more nuanced behaviour on weighted data.
When a clustering satisfies a reasonable notion of clusterability, average-linkage detects that clustering irrespective of weights. On the other hand, this algorithm responds to weights on all other clusterings. We note that the notion of clusterability required for average-linkage is much weaker than the notion of clusterability used to characterize the behaviour of ratio-cut on weighted data.

%Hierarchical algorithms output dendrograms, which contain multiple clusterings. Please see the %preliminary section for definitions relating to the hierarchical setting. Weight-responsive for %hierarchical algorithms is defined analogously to Definition~\ref{weight responsive}.

\subsection{Average Linkage}

Linkage-based algorithms start by placing each element in its own cluster, and proceed by repeatedly merging the ``closest'' pair of clusters until the entire dendrogram
is constructed.  To identify the closest clusters, these algorithms use a
linkage function that maps pairs of clusters to a real number. Formally, a \emph{linkage function}  is a function \begin{small}$\ell:
\{(X_1, X_2,d,w) \mid d,w \textrm{ over } X_1\cup X_2\} \rightarrow R^+.$\end{small}

Average-linkage is one of the most popular linkage-based algorithms (commonly applied in
bioinformatics under the name Unweighted Pair Group Method with Arithmetic Mean). Recall that $w(X) = \sum_{x \in X}w(x)$. The average-linkage
linkage function is $$\ell_{AL}(X_1, X_2,d,w) =
\frac{\sum_{x \in X_1, y\in X_2}  d(x, y) \cdot w(x) \cdot w(y)}{w(X_1)\cdot w(X_2)}.$$

To study how average-linkage responds to weights, we give a relaxation of the notion of a perfect clustering.

\begin{definition}[Nice]
A clustering $C$ of $(w[X],d)$ is \emph{nice} if for all $x_1,x_2,x_3 \in X$ where $x_1 \sim_C x_2$ and $x_1 \not\sim_C x_3$, $d(x_1,x_2) < d(x_1,x_3)$.
\end{definition}

Data sets with nice clusterings correspond to those that satisfy the ``strict separation'' property
introduced by Balcan \emph{et al.} \cite{balcan}. As for a perfect clustering, being a nice clustering is independent of weights. Note that all perfect clusterings are nice, but not all nice clusterings are perfect.
A dendrogram is \emph{nice} if all clusterings that appear in it are nice.

%A hierarchical algorithm $\mathcal{A}$ is
%\emph{nice detecting} if for all $(w[X],d)$, and every nice clustering $C$ of $(w[X],d)$, $C$ is in
%$\mathcal{A}(w[X],d)$.

We present a complete characterisation of the way that average-linkage (AL) responds to weights.

\begin{theorem}\label{average linkage}
Given $(X,d)$, $|range(AL(X,d))|=1$ if and only if $(X,d)$ has a nice dendrogram.
\end{theorem}
\begin{proof}
We first show that if a data set has a nice dendrogram, then this is the dendrogram that average-linkage outputs. Note that the property of being nice is independent of the weight function. So, the set of nice clusterings of any data set $(w[X],d)$ is invariant to the weight function $w$.
Lemma~\ref{Average nice detecting} shows that, for every $(w[X],d)$, every nice clustering in $(w[X],d)$ appears in the dendrogram produce by average-linkage.

Let $(X,d)$ be a data set that has a nice dendrogram $D$. We would like to show that average-linkage outputs that dendrogram.
Let $\mathcal{C}$ be the set of all nice clusterings of $(X,d)$. Let $\mathcal{L} = \{c \mid \exists C \in \mathcal{C} \textrm{ such that } c \in C\}.$ That is, $\mathcal{L}$ is the set of all clusters that appear in some nice clustering of $(X,d)$.

Since $D$ is a nice dendrogram of $(X,d)$, all clusterings that appear in it are nice, and so it contains all the clusters in $\mathcal{L}$ and no additional clusters. In order to satisfy the condition that every nice clustering of $(X,d)$ appears in the dendrogram, $D_{AL}$, produced by average-linkage, $D_{AL}$ must have all clusters in $\mathcal{L}$.

Since a dendrogram is a strictly binary tree, any dendrogram of $(X,d)$ has exactly $|X|-1$ inner nodes. In particular, all dendrograms of the same data set have exactly the same number of inner nodes. This implies that $D_{AL}$ has the same clusters as $D$, so $D_{AL}$ is equivalent to $D$.

%NOTE!!!: DEFINE EQUIVALENT RELATIONS BETWEEN DENDROGRAMS?

Now, let $(X,d)$ be a data set that does not have a nice dendrogram. Then, given any $w$ over $X$, there is a clustering $C$ that is not nice that appears in $AL(w[X],d)$.
Lemma~\ref{AL on not nice} shows that if a clustering that is not nice appears in $AL(w[X],d)$, then $|range(AL(X,d))|>1$.

\end{proof}

Theorem~\ref{average linkage} follows from the two lemmas below.

\begin{lemma}\label{Average nice detecting}
Given any weighted data set $(w[X],d)$, if $C$ is a nice clustering of $(X,d)$, then $C$ is in
the dendrogram produced by average-linkage on $(w[X],d)$.
\end{lemma}
\begin{proof}
Consider a nice clustering $C = \{C_1, \ldots, C_k\}$ over $(w[X],d)$.
It suffices to show that for any $1 \leq i < j \leq k$,  $X_1, X_2 \subseteq C_i$ where $X_1 \cap X_2 = \emptyset$ and $X_3 \subseteq C_j$,
$\ell_{AL}(X_1, X_2,d,w) < \ell_{AL}(X_1, X_3,d,w)$.
It can be shown that $$\ell_{AL}(X_1,X_2,d,w) \leq \frac{\sum_{x_1 \in X_1} w(x_1)\cdot\max_{x_2 \in X_2}d(x_1,x_2)}{w(X_1)}$$
and $$\ell_{AL}(X_1,X_3,d, w) \geq \frac{\sum_{x_1 \in X_1} w(x_1)\cdot\min_{x_3 \in X_3}d(x_1,x_3)}{w(X_1)}.$$
Since $C$ is nice, it follows that
$$\min_{x_3 \in X_3}d(x_1,x_3)> \max_{x_2 \in X_2}d(x_1,x_2)$$
Thus $\ell_{AL}(X_1,X_3)> \ell_{AL}(X_1, X_2)$, which completes the proof.
\end{proof}

\begin{lemma}\label{AL on not nice}
For any data set $(X,d)$ and any weight function $w$ over $X$, if a clustering that is not nice appears in $AL(w[X],d)$, then $|range(AL(X,d))|>1$.
\end{lemma}
\begin{proof}
Let $(X,d)$ be a data set so that a clustering $C$ that is not nice appears in $AL(w[X],d)$ for some weight function $w$ over $X$.  We construct $w'$ so that $C \not\in AL(w'[X],d)$, which would show that $|range(AL(X,d))|>1$.

Since $C$ is not nice, there exist $1 \leq i,j \leq k$, $i \neq j$, and $x_1, x_2 \in C_i$, $x_1 \neq x_2$, and $x_3 \in C_j$, so that
$d(x_1,x_2) > d(x_1,x_3)$.

Now, define weight function $w'$ as follows: $w'(x) = 1$ for all $x \in X \setminus \{x_1,x_2, x_3\}$, and
$w'(x_1) = w'(x_2) =  w'(x_3) = W$, for some large value $W$.  We argue that when $W$ is sufficiently large,
$C$ is not a clustering in $AL(w'[X],d)$.

By way of contradiction, assume that $C$ is a clustering in $AL(w'[X],d)$ for
any setting of $W$. Then there is a step in the algorithm where clusters $X_1$
and $X_2$ merge, where $X_1, X_2 \subset C_i$, $x_1 \in X_1$, and $x_2 \in X_2$. At
this point, there is some cluster $X_3 \subseteq C_j$ so that $x_3 \in X_3$.

We compare $\ell_{AL}(X_1,X_2, d, w')$ and $\ell_{AL}(X_1, X_3, d,w')$. First, note that
\[a
\ell_{AL}(X_1,X_2, d, w') = \frac{W^2d(x_1, x_2) + \alpha_1W+\alpha_2}{W^2+\alpha_3W + \alpha_4}
\]
for some non-negative real valued $\alpha_i$s. Similarly, we have that for some non-negative real-valued $\beta_i$:
\[
\ell_{AL}(X_1, X_3, d,w')=\frac{W^2d(x_1, x_3)+ \beta_1W+\beta_2}{W^2+\beta_3W + \beta_4}
\]

Dividing by $W^2$, we see that $\ell_{AL}(X_1, X_3, d,w') \rightarrow d(x_1, x_3)$ and
$\ell_{AL}(X_1,X_2, d, w') \rightarrow d(x_1, x_2)$ as $W \to \infty$, and so the result holds since $d(x_1, x_3) <  d(x_1, x_2)$. Therefore average linkage merges $X_1$ with $X_3$, thus cluster $C_i$ is never formed, and so $C$ is not a clustering in $AL(w'[X],d)$. If follows that $|range(AL(X,d))|>1$,
\end{proof}

%Finally, note that average-linkage outputs all nice clusterings present in a data set, %regardless
%of weights.

%%%%%%%%%%%%%%%%%%%%%%%%%%%%%%%%%%%%%%%%%%%%%%%%%%%%%%%%%%%%%%%%%%%%%%%%%%%%%%%%%%%%%%%%%%%%%%%%%%%%%%%%%%%%%%%%%%%%%%%%%%%%%%%%%%
\subsection{Ward's Method}

Ward's method is a highly effective clustering algorithm  (\cite{ward1963hierarchical}), which, at every step, merges
the clusters that will yield the minimal increase to the sum-of-squares error (the $k$-means objective function). Let $ctr(X,d,w)$ be the
center of mass of the data set $(w[X],d)$.  Then, the linkage function for Ward's method is
$$ \ell_{Ward}(X_1, X_2, d,w)
 = \frac{w(X_1)\cdot w(X_2) \cdot d( ctr(X_1,d,w), ctr(X_2,d,w))^2}{w(X_1) + w(X_2)},$$ where $X_1$ and $X_2$ are disjoint subsets (clusters) of $X$.

\begin{theorem}\label{Ward}
Ward's method is weight sensitive.
\end{theorem}
\begin{proof}
Consider any data set $(X,d)$ and any clustering $C$ output by Ward's method on $(X,d)$. Let $x,y
\in X$ be any distinct points that belong to the same cluster in $C$. Let $w$ be the weight function
that assigns a large weight $W$ to points $x$ and $y$, and weight $1$ to all other elements.

Since Ward's method is linkage-based, it starts off by placing every element in its own cluster. We
will show that when $W$ is large enough, it is prohibitively expensive to merge a cluster that
contains $x$ with a cluster that contains point $y$. Therefore, there is no cluster in the
dendrogram produced by Ward's method that contains both points $x$ and $y$, other than the root, and
so $C$ is not a clustering in that dendrogram. This would imply that $|range(Ward(X,d))|>1$.

At some point in the execution of Ward's method, $x$ and $y$ must belong to different clusters.  Let
$C_i$ be a cluster that contains $x$, and cluster $C_j$ a cluster that contains point $y$. Then
$\ell_{Ward}(C_i, C_j,d,w) \rightarrow \infty$ as $W \rightarrow \infty$. On the other hand,
whenever at most one of $C_i$ or $C_j$ contains an element of $\{x,y\}$, $\ell_{Ward}(C_i,C_j,d,w)$
approaches some constant as $W \rightarrow \infty$. This shows that when $W$ is sufficiently large,
a cluster containing $x$ is merged with a cluster containing $y$ only at the last step of the
algorithm, when forming the root of the dendrogram.

\end{proof}

%%%%%%%%%%%%%%%%%%%%%%%%%%%%%%%%%%%%%%%%%%%%%%%%%%%%%%%%%%%%%%%%%%%%%%%%%%%%%%%%%%%%%%%%%%%%%%%%%%%%%%%%%%%%%%%%%%%%%%%%%%%%%%%%%%
%%%%%%%%%%%%%%%%%%%%%%%%%%%%%%%%%%%%%%%%%%%%%%%%%%%%%%%%%%%%%%%%%%%%%%%%%%%%%%%%%%%%%%%%%%%%%%%%%%%%%%%%%%%%%%%%%%%%%%%%%%%%%%%%%%

\subsection{Divisive Algorithms}

The class of divisive clustering algorithms is a well-known family of hierarchical algorithms, which
construct the dendrogram by using a top-down approach. This family of algorithms
includes the popular bisecting $k$-means algorithm. We show that a class of algorithms that includes bisecting $k$-means consists of weight-sensitive methods.

Given a node $x$ in dendrogram $(T,M)$, let $\mathcal{C}(x)$ denote the cluster represented by node
$x$. That is, $\mathcal{C}(x) = \{M(y) \mid y \textrm{ is a leaf and a descendent of }x\}$.

Informally, a $\mathcal{P}$-Divisive algorithm is a hierarchical clustering algorithm that uses a
partitional clustering algorithm $\mathcal{P}$ to recursively divide the data set into two clusters
until only single elements remain. Formally, a $\mathcal{P}$-divisive algorithm is defined as follows.

\begin{definition}[$\mathcal{P}$-Divisive]\label{divisive}

A hierarchical clustering algorithm $\mathcal{A}$ is \emph{$\mathcal{P}$-Divisive} with respect to a partitional
clustering algorithm $\mathcal{P}$, if for all $(X,d)$, we have $\mathcal{A}(w[X],d) = (T,M)$, such that for
all non-leaf nodes $x$ in $T$ with children $x_1$ and $x_2$,  $\mathcal{P}(w[\mathcal{C}(x)],d,2) =
\{\mathcal{C}(x_1), \mathcal{C}(x_2)\}$.

\end{definition}

We obtain bisecting $k$-means by setting $\mathcal{P}$ to $k$-means. Other natural choices for  $\mathcal{P}$ include min-sum, and exemplar-based algorithms such as $k$-median. As shown above, many of these partitional algorithms are weight-separable. We show that whenever $\mathcal{P}$ is weight-separable, then $\mathcal{P}$-Divisive is weight-sensitive.

\begin{theorem}
If $\mathcal{P}$ is weight-separable then the $\mathcal{P}$-Divisive algorithm is weight-sensitive.
\end{theorem}
\begin{proof}
Given any non-trivial clustering $C$ output by the $\mathcal{P}$-Divisive algorithm, consider any pair of
elements $x$ and $y$ that are placed within the same cluster of $C$. Since $\mathcal{P}$ is weight
separating, there exists a weight function $w$ so that $\mathcal{P}$ separates points $x$ and $y$.
Then $\mathcal{P}$-Divisive splits $x$ and $y$ in the first step, directly below the root, and
clustering $C$ is never formed.
\end{proof}

\section{Heuristic Approaches}\label{heuristics}

We have seen how weights affect various algorithms that optimize different
clustering objectives.  Since optimizing a clustering objective is usually
NP-hard, heuristics are used in practice. In this section, we  consider several common
heuristical clustering approaches, and show how they respond to
weights.

We note that there are many algorithms that aim to find high quality partitions for popular objective functions. For the $k$-means objective alone many different algorithms have been proposed, most of which provide different initializations for Lloyd's method. For example, \cite{steinley2007initializing} studied  a dozen different initializations. There are many other algorithms based on the $k$-means objective functions, some of the most notable being  $k$-means++ (\cite{Arthur}) and the Hochbaum-Schmoys initialization (\cite{hochbaum1985best}) studied, for instance, by \cite{bubeck2009initialization}. As such, this section is not intended as a comprehensive analysis of all available heuristics, but rather it shows how to analyze such heuristics, and provides a classification for some of the most popular approaches. 

To define the categories in the randomized setting, we need to modify the definition of \emph{range}. Given a randomized, partitional clustering algorithm $\mathcal{A}$ and a data set $(X,d)$, the \emph{randomized range} is $$randRange(\mathcal{A}(X,d)) = \{ C \mid \forall \epsilon<1 \exists w \textrm{ such that } P( \mathcal{A}(w[X],d) = C) > 1-\epsilon\}.$$ That is, the randomized range is the set of clusterings that are produced with arbitrarily high probably, when we can modify weights.

The categories describing an algorithm's behaviour on weighted data are defined as previously, but using randomized range.

\begin{definition}[Weight Sensitive (Randomized Partitional )]
A partitional algorithm $\mathcal{A}$ is \emph{weight-sensitive} if for all $(X,d)$ and $1 < k < |X|$, $|randRange(\mathcal{A}(X,d,k))| > 1$.
\end{definition}

\begin{definition}[Weight Robust (Randomized Partitional)]
A partitional algorithm $\mathcal{A}$ is \emph{weight-robust} if for all $(X,d)$ and $1 < k < |X|$, $|randRange(\mathcal{A}(X,d,k))| =1$.
\end{definition}

\begin{definition}[Weight Considering (Randomized Partitional)]
A partitional algorithm $\mathcal{A}$ is \emph{weight-considering} if
\begin{itemize}
\item There exist $(X,d)$ and $1< k <|X|$ so that $|randRange(\mathcal{A}(X,d,k))| = 1$, and
\item There exist $(X,d)$ and $1< k <|X|$ so that $|randRange(\mathcal{A}(X,d,k))| > 1$.
\end{itemize}
\end{definition}

%%%%%%%%%%%%%%%%%%%%%%%%%%%%%%%%%%%%%%%%%%%%%%%%%%%%%%%%%%%%%%%%%%%%%%%%%%%%%%%%%%%%%%%%%%%%%%%%%%%%%%%%%%%%%%%%%%%%%%%%%%%%%%%%%%
%%%%%%%%%%%%%%%%%%%%%%%%%%%%%%%%%%%%%%%%%%%%%%%%%%%%%%%%%%%%%%%%%%%%%%%%%%%%%%%%%%%%%%%%%%%%%%%%%%%%%%%%%%%%%%%%%%%%%%%%%%%%%%%%%%
\subsection{Partitioning Around Medoids (PAM)}

In contrast with the Lloyd method and $k$-means, PAM is a heuristic for exemplar-based objective functions such as $k$-medoids, which chooses data points as centers (thus it is not required to compute centers of mass).  As a result, this approach can be applied to arbitrary data, not only to normed vector spaces.

\emph{Partitioning around medoids} (PAM) is given an initial set of
$k$ centers, $T$, and changes $T$ iteratively to find a ``better'' set
of $k$ centers.  This is done by swapping out centers for other points
in the data set and computing the cost function: $\sum_{c \in T}
\sum_{x \in X \setminus T, c \sim_C x} d(x,c)\cdot w(x)$.  Each
iteration performs a swap only if a better cost is possible, and it stops when
no changes are made \cite{PAM_book}.

Note that our results in this section hold regardless of how the initial $k$
centers are chosen from $X$.

\begin{theorem}
PAM is weight-separable.
\end{theorem}
\begin{proof}
Let $T = \{x_1, \ldots, x_l\}$ be $l$ points that we want to separate, where $2 \leq l \leq k$.
Let $\{m_1, \ldots, m_k\} \subset X$ be $k$ centers chosen by the PAM initialization, and denote by
$C = \{c_1, \ldots, c_k\}$ the clustering induced by the corresponding centers.  Set $w(x_i) = W,
\forall x_i \in S$ for some large $W$. We first note that any optimal clustering $C^{*}$ sets the
points in $T$ as the centers. The cost of $C^{*}$ is constant as a function of $W$, while every
clustering with a different set of centers has a cost proportional to $W$, which can be made
arbitrarily high by increasing $W$.

Assume, by contradiction, that the algorithm stops at a clustering $C$ that does not separate all
the points in $T$. Then, there exists a cluster $c_i \in C$ such that $|c_i \cap T| \geq 2$.
Thus, $c_i$ contributes a factor of $\alpha\cdot W$ to the cost, for some $\alpha > 0$.  Further,
there exists a cluster $c_j \in C$ such that $c_j \cap T = \emptyset$.  Then the cost of $C$ can
be further decreased, by a quantity proportional to $W$, by assigning one of the heavy non-medoid
from $c_i$ to be the center of $c_j$, which is a contradiction. Thus, the algorithm cannot stop
before setting all the heavy points as cluster centers.
\end{proof}

\begin{comment}
\begin{thm}
$k$-means++ is weight-separable.
\end{thm}
\begin{proof}
If any $k$ points $\{x_1, \ldots, x_k\}$ are assigned sufficiently high weight $W$, then the first
center will be one of these points with arbitrarily high probability. The next center will also be
selected with arbitrarily high probability if $W$ is large enough, since for all $y \not\in \{x_1,
\ldots, x_k\}$, the probability of selecting $y$ can be made arbitrarily small when $W$ is large
enough.
\end{proof}
\end{comment}

%%%%%%%%%%%%%%%%%%%%%%%%%%%%%%%%%%%%%%%%%%%%%%%%%%%%%%%%%%%%%%%%%%%%%%%%%%%%%%%%%%%%%%%%%%%%%%%%%%%%%%%%%%%%%%%%%%%%%%%%%%%%%%%%%%
%%%%%%%%%%%%%%%%%%%%%%%%%%%%%%%%%%%%%%%%%%%%%%%%%%%%%%%%%%%%%%%%%%%%%%%%%%%%%%%%%%%%%%%%%%%%%%%%%%%%%%%%%%%%%%%%%%%%%%%%%%%%%%%%%%
\subsection{Llyod method}

The Lloyd method is a heuristic commonly used for uncovering clusterings with low $k$-means objective
cost.  The Lloyd algorithm can be combined with different approaches for seeding the initial
centers.  In this section, we start by considering the following deterministic seeding methods.

\begin{definition}[Lloyd Method]
Given $k$ points (centers) $\{c_1, \ldots, c_k\}$ in the space, assign every element of $X$ to its closest center.
Then compute the centers of mass of the resulting clusters by summing the elements in each cluster and dividing by the number of elements in that partition, and assign every element to its closest new center. Continue until
no change is made in one iteration.
\end{definition}

With the Lloyd method, the dissimilarity (or, distance) to the center can be both the $\ell_1$-norm or squared.

First, we consider the case when the $k$ initial centers are chosen in a deterministic fashion.
For example, one deterministic seeding approach involves selecting the $k$-furthest centers (see, for example, \cite{NIPS2010}).

\begin{theorem}
\label{thm:det-lloyd-not-nice}
Let $\mathcal{A}$ represent the Lloyd method with some deterministic seeding procedure. Consider any data set $(X,d)$ and $1< k<|X|$.
 If there exists a clustering $C \in range(\mathcal{A}(X,d,k))$ that is not nice, then $|range(\mathcal{A}(X,d,k))|>1$.
\end{theorem}
\begin{proof}
For any seeding procedure, since $C$ is in the range of $\mathcal{A}(X,d)$, there exists a weight
function $w$ so that $C = \mathcal{A}(w[X],d)$.

Since $C$ is not nice, there exist points $x_1, x_2, x_3 \in X$ where $x_1 \sim_C x_2$, $x_1
\not\sim_C x_3$, but $d(x_1,x_3) < d(x_1,x_2)$.  Construct weight function $w'$ such that $w'(x) =
1$ for all $x \in X \setminus \{x_2,x_3\}$, and $w'(x_2) = w'(x_3) = W$, for some constant $W$.

If for some value of $W$, $\mathcal{A}(w'[X],d,k)\neq C$, then we're done. Otherwise,
$\mathcal{A}(w'[X],d,k)=C$ for all values of $W$. But if $W$ is large enough, the center of mass of
the cluster containing $x_1$ and $x_2$ is arbitrarily close to $x_2$, and the center of mass of the
cluster containing $x_3$ is arbitrarily close to $x_3$. But since $d(x_1,x_3) < d(x_1,x_2)$, the Lloyd method
would assign $x_1$ and $x_3$ to the same cluster. Thus, when $W$ is sufficiently large,
$\mathcal{A}(w'[X],d,k)\neq C$.
\end{proof}

We also show that for a deterministic, weight-independent initialization, if the Lloyd method
outputs a nice clustering $C$, then this algorithm is robust to weights on that data. 

\begin{theorem}
\label{thm:det-lloyd-nice}
Let $\mathcal{A}$ represent the Lloyd method
with some weight-independent deterministic seeding procedure.  Given $(X,d)$, if there exists a nice clustering in the $range(\mathcal{A}(X,d))$, then $\mathcal{A}$ is weight robust on $(X,d)$. 
\end{theorem}
\begin{proof}
Since the initialization is weight-independent, $A$ will find the same initial centers on any weight
function. Given a nice clustering, the Lloyd method does not modify the clustering.  If the seeding
method were not weight-independent, it may seed in a way that may prevent the Lloyd method from
finding $C$ for some weight function.
\end{proof}

\begin{corollary}
Let $\mathcal{A}$ represent the Lloyd method initialized with furthest centroids.  For any $(X,d)$ and $1 < k < |X|$, $|range(\mathcal{A}(X,d,k))|=1$ if and only if there exists a nice $k$-clustering of $(X,d)$.
\end{corollary}

%%%%%%%%%%%%%%%%%%%%%%%%%%%%%%%%%%%%%%%%%%%%%%%%%%%%%%%%%%%%%%%%%%%%%%%%%%%%%%%%%%%%%%%%%%%%%%%%%%%%%%%%%%%%%%%%%%%%%%%%%%%%%%%%%

\subsubsection{$k$-means++}

The $k$-means++ algorithm, introduced by Arthur and Vassilvitskii (\cite{Arthur}) is the Lloyd algorithm with a randomized
initialization method that aims to place the initial centers far apart from each other. This algorithm has been demonstrated to perform very well in practice.

Let $D(x)$ denote the shortest dissimilarity from a point $x$ to the closest center already chosen.
The $k$-means++ algorithm chooses the initial center uniformly at random, and then $x$ is selected as the next center with probability $\frac{D(x)^2w(x)}{\sum_{y}D(y)^2w(y) }$ until $k$ centers have been chosen.

\begin{theorem}
$k$-means++ is weight-separable.
\end{theorem}
\begin{proof}
If any $k$ points $\{x_1, \ldots, x_k\}$ are assigned sufficiently high weight $W$, then the first
center will be one of these points with arbitrarily high probability. The next center will also be
selected with arbitrarily high probability if $W$ is large enough, since for all $y \not\in \{x_1,
\ldots, x_k\}$, the probability of selecting $y$ can be made arbitrarily small when $W$ is large
enough.
\end{proof}

The same argument works for showing that the Lloyd method is weight-separable when the initial centers are selected uniformly at random (Randomized Lloyd). An expanded classification of clustering algorithms that includes heuristics is given in Table 1 below.

%%%%%%%%%%%%%%%%%%%%%%%%%%%%%%%%%%%%%%%%%%%%%%%%%%%%%%%%%%%%%%%%%%%%%%%%%%%%%%%%%%%%%%%%%%%%%%%%%%%%%%%%%%%%%%%%%%%%%%%%%%%%%%%%%%
%%%%%%%%%%%%%%%%%%%%%%%%%%%%%%%%%%%%%%%%%%%%%%%%%%%%%%%%%%%%%%%%%%%%%%%%%%%%%%%%%%%%%%%%%%%%%%%%%%%%%%%%%%%%%%%%%%%%%%%%%%%%%%%%%%

\begin{table*}[ht]
\begin{center}
\begin{tabular}[ht!]{||l||l|l|l||}
\hline
\hline
 & {\textbf{Partitional}}  & {\textbf{Hierarchical}} &  {\textbf{Heuristics}}\\
\hline
\hline
\textbf{Weight} &  $k$-means, $k$-medoids& Ward's method & Randomized Lloyd,
\\
\textbf{Sensitive}& $k$-median, Min-sum& Bisecting $k$-means & PAM, $k$-means++\\
%&\emph{Exemplar-based}  & \emph{$\mathcal{F}$-Divisive}\footnote{$\mathcal{F}$-Divisive is weight-sensitive whenever $\mathcal{F}$ is weight-separable.} & \\
 \hline
\textbf{Weight} &  &  & Lloyd with \\
\textbf{Considering} & Ratio-cut & Average-linkage & Furthest centroids\\
%  &  &  & \emph{Deterministic Lloyd}\footnote{Deterministic Lloyd is weight considering whenever the initialization is weight-independent.}\\
 \hline
\textbf{Weight} & Min-diameter & Single-linkage &   \\
\textbf{Robust}   & $k$-center & Complete-linkage & \\
\hline
\hline
\end{tabular}
\end{center}
\caption{A classification of clustering algorithms based on their response to weighted data expanded to include several popular heuristic methods.}
\label{summary_full}
\end{table*}

%%%%%%%%%%%%%%%%%%%%%%%%%%%%%%%%%%%%%%%%%%%%%%%%%%%%%%%%%%%%%%%%%%%%%%%%%%%%%%%%%%%%%%%%%%%%%%%%%%%%%%%%%%%%%%%%%%%%%%%%%%%%%%%%%%
%%%%%%%%%%%%%%%%%%%%%%%%%%%%%%%%%%%%%%%%%%%%%%%%%%%%%%%%%%%%%%%%%%%%%%%%%%%%%%%%%%%%%%%%%%%%%%%%%%%%%%%%%%%%%%%%%%%%%%%%%%%%%%%%%%
%%%%%%%%%%%%%%%%%%%%%%%%%%%%%%%%%%%%%%%%%%%%%%%%%%%%%%%%%%%%%%%%%%%%%%%%%%%%%%%%%%%%%%%%%%%%%%%%%%%%%%%%%%%%%%%%%%%%%%%%%%%%%%%%%%

\section{Conclusion}

We studied the
behaviour of clustering algorithms on weighted data, presenting three fundamental categories that
describe how such algorithms respond to weights and classifying several
well-known algorithms according to these categories. Our results are summarized in Table~\ref{summary}. We note that all of our results immediately translate to the standard setting, by mapping each point with integer weight to the same number of unweighted duplicates.

Our results can be used to aid in the selection of a clustering algorithm. For example, in the facility allocation application discussed in the introduction, where weights are of primal importance, a weight-sensitive algorithm is suitable. Other applications may call for weight-considering algorithms. This can
occur when weights (i.e. number of duplicates) should not be ignored, yet
it is still desirable to identify rare instances that constitute small but well-formed
outlier clusters. For example, this applies to patient data on potential
causes of a disease, where it is crucial to investigate rare instances. 

This paper presents a significant step forward in the property-based approach for selecting clustering algorithms. Unlike previous properties, which focused on advantages of linkage-based algorithms, these properties show when applications call for popular center-based approaches, such as $k$-means. Furthermore, the simplicity of these properties makes them widely applicable, requiring only that the user decide whether duplicating elements should be able to change the output of the algorithm. Future work will consider complimentary considerations, with the ultimate goal of attaining a small set of properties that will aid in ``the user's dilemma'' for a wide range of clustering applications. 

}

\bibliographystyle{plain}
%\bibliography{arxivbib}

\begin{thebibliography}{10}
	
	\bibitem{IJCAI2011}
	M.~Ackerman and S.~Ben-David.
	\newblock Discerning linkage-based algorithms among hierarchical clustering
	methods.
	\newblock In {\em IJCAI}, 2011.
	
	\bibitem{COLT2010}
	M.~Ackerman, S.~Ben-David, and D.~Loker.
	\newblock Characterization of linkage-based clustering.
	\newblock In {\em COLT}, 2010.
	
	\bibitem{NIPS2010}
	M.~Ackerman, S.~Ben-David, and D.~Loker.
	\newblock Towards property-based classification of clustering paradigms.
	\newblock In {\em NIPS}, 2010.
	
	\bibitem{Arthur}
	D.~Arthur and S.~Vassilvitskii.
	\newblock K-means++: The advantages of careful seeding.
	\newblock In {\em SODA}, 2007.
	
	\bibitem{balcan}
	M.~F. Balcan, A.~Blum, and S.~Vempala.
	\newblock A discriminative framework for clustering via similarity functions.
	\newblock In {\em STOC}, 2008.
	
	\bibitem{reza}
	R.~Bosagh-Zadeh and S.~Ben-David.
	\newblock A uniqueness theorem for clustering.
	\newblock In {\em UAI}, 2009.
	
	\bibitem{bubeck2009initialization}
	S{\'e}bastien Bubeck, Marina Meila, and Ulrike von Luxburg.
	\newblock How the initialization affects the stability of the k-means
	algorithm.
	\newblock {\em arXiv preprint arXiv:0907.5494}, 2009.
	
	\bibitem{carlsson2010characterization}
	Gunnar Carlsson and Facundo M{\'e}moli.
	\newblock Characterization, stability and convergence of hierarchical
	clustering methods.
	\newblock {\em The Journal of Machine Learning Research}, 11:1425--1470, 2010.
	
	\bibitem{fisher}
	L.~Fisher and J.~Van Ness.
	\newblock Admissible clustering procedures.
	\newblock {\em Biometrika}, 58:91--104, 1971.
	
	\bibitem{hochbaum1985best}
	Dorit~S Hochbaum and David~B Shmoys.
	\newblock A best possible heuristic for the k-center problem.
	\newblock {\em Mathematics of operations research}, 10(2):180--184, 1985.
	
	\bibitem{hubert1975hierarchical}
	Lawrence Hubert and James Schultz.
	\newblock Hierarchical clustering and the concept of space distortion.
	\newblock {\em British Journal of Mathematical and Statistical Psychology},
	28(2):121--133, 1975.
	
	\bibitem{jardine1968construction}
	Nicholas Jardine and Robin Sibson.
	\newblock The construction of hierarchic and non-hierarchic classifications.
	\newblock {\em The Computer Journal}, 11(2):177--184, 1968.
	
	\bibitem{PAM_book}
	L.~Kaufman and P.~J. Rousseeuw.
	\newblock {\em Partitioning Around Medoids (Program PAM)}, pages 68--125.
	\newblock John Wiley \& Sons, Inc., 2008.
	
	\bibitem{Kleinberg}
	J.~Kleinberg.
	\newblock An impossibility theorem for clustering.
	\newblock {\em Proceedings of International Conferences on Advances in Neural
		Information Processing Systems}, pages 463--470, 2003.
	
	\bibitem{ostrovsky}
	R.~Ostrovsky, Y.~Rabani, L.~J. Schulman, and C.~Swamy.
	\newblock {The effectiveness of Lloyd-type methods for the k-means problem}.
	\newblock In {\em FOCS}, 2006.
	
	\bibitem{sahni1976p}
	Sartaj Sahni and Teofilo Gonzalez.
	\newblock P-complete approximation problems.
	\newblock {\em Journal of the ACM (JACM)}, 23(3):555--565, 1976.
	
	\bibitem{steinley2006k}
	D.~Steinley.
	\newblock K-means clustering: a half-century synthesis.
	\newblock {\em British Journal of Mathematical and Statistical Psychology},
	59(1):1--34, 2006.
	
	\bibitem{steinley2007initializing}
	Douglas Steinley and Michael~J Brusco.
	\newblock Initializing k-means batch clustering: a critical evaluation of
	several techniques.
	\newblock {\em Journal of Classification}, 24(1):99--121, 2007.
	
	\bibitem{von2007tutorial}
	U.~{Von Luxburg}.
	\newblock {A tutorial on spectral clustering}.
	\newblock {\em J. Stat. Comput.}, 17(4):395--416, 2007.
	
	\bibitem{ward1963hierarchical}
	Joe~H Ward~Jr.
	\newblock Hierarchical grouping to optimize an objective function.
	\newblock {\em Journal of the American statistical association},
	58(301):236--244, 1963.
	
	\bibitem{wright}
	W.~E. Wright.
	\newblock {A formalization of cluster analysis}.
	\newblock {\em J. Pattern Recogn.}, 5(3):273--282, 1973.
	
\end{thebibliography}

\end{document}